\documentclass[twoside]{article}

\usepackage[accepted]{aistats2020}

\usepackage{amsmath,amsfonts,bm}

\def\eqref#1{equation~\ref{#1}}

\def\1{\bm{1}}

\def\vc{{\bm{c}}}

\def\vw{{\bm{w}}}
\def\vx{{\bm{x}}}
\def\vy{{\bm{y}}}

\def\mI{{\bm{I}}}

\def\mK{{\bm{K}}}

\def\mW{{\bm{W}}}

\DeclareMathAlphabet{\mathsfit}{\encodingdefault}{\sfdefault}{m}{sl}
\SetMathAlphabet{\mathsfit}{bold}{\encodingdefault}{\sfdefault}{bx}{n}

\DeclareMathOperator*{\argmin}{arg\,min}

\usepackage[breaklinks]{hyperref}
\usepackage{url}
\usepackage{subfloat}
\usepackage[top=0.5in]{geometry}
\usepackage{fullpage}
\usepackage{graphicx}
\usepackage{epstopdf}
\usepackage{float}
\usepackage{subfloat}
\usepackage{caption}
\usepackage{breakcites}

\usepackage{graphicx}
\usepackage{subfigure}

\usepackage{array}
\usepackage{booktabs}
\usepackage{multirow}

\usepackage{amsmath,amssymb,amsfonts,bm}
\usepackage{amsthm,amsfonts}
\usepackage{mathrsfs}
\usepackage{empheq}
\newtheorem{theorem}{Theorem}

\usepackage{mathtools}

\usepackage{algorithm}
\usepackage{algorithmic}

\newtheorem{lemma}{Lemma}
\newtheorem{assump}{Assumption}
\newtheorem{coro}{Corollary}
\newtheorem{prop}{Proposition}

\begin{document}
	
	%
	
	%
	
	{\twocolumn[\aistatstitle{Robust Importance Weighting for  Covariate Shift}
		\aistatsauthor{  Henry Lam 
			\And Fengpei Li 
			\And  Siddharth Prusty 
		}
		\aistatsaddress{ Columbia University 
			\\ khl2114@columbia.edu  
			\And Columbia University 
			
			\\ Email fl2412@columbia.edu
			 \And  Columbia University
			  \\ siddharth.prusty@columbia.edu 
		 } ]}
	
	\begin{abstract}
		In many learning problems, the training and testing data follow different distributions and a particularly common situation is the \textit{covariate shift}. To correct for sampling biases, most approaches, including the popular kernel mean matching (KMM), focus on estimating the importance weights between the two distributions. Reweighting-based methods, however, are exposed to high variance when the distributional discrepancy is large and the weights are poorly estimated. On the other hand, the alternate approach of using nonparametric regression (NR) incurs high bias when the training size is limited. In this paper, we propose and analyze a new estimator that systematically integrates the residuals of NR with KMM reweighting, based on a  control-variate perspective. The proposed estimator can be shown to either strictly outperform or match the best-known existing rates for both KMM and NR, and thus is a robust combination of both estimators. The experiments shows the estimator works well in practice.
	\end{abstract}
	
	\section{Introduction}
	Traditional machine learning implicitly assumes training and test data are drawn from the same distribution. However, mismatches between training and test distributions occur frequently in reality. For example, in clinical trials the patients used for prognostic factor identification may not come from the target population due to sample selection bias \cite{huang2007correcting,gretton2009covariate}; incoming signals used for natural language and image processing, bioinformatics or econometric analyses change in distribution over time and seasonality \cite{heckman1979sample,zadrozny2004learning,sugiyama2007covariate,quionero2009dataset, tzeng2017adversarial, jiang2007instance,borgwardt2006integrating}; patterns for engineering controls fluctuate due to the non-stationarity of environments \cite{sugiyama2012machine, hachiya2008adaptive}. 
	
	Many such problems are investigated under the \textit{covariate shift} assumption \cite{shimodaira2000improving}. Namely, in a supervised learning setting with covariate $X$ and label $Y$, the marginal distribution of $X$ in the training set $P_{tr}(x)$, shifts away from the marginal distribution of the test set $P_{te}(x)$, while the conditional distribution $P(y|x)$ remains invariant in both sets. Because test labels are either too costly to obtain or unobserved, it could be uneconomical or impossible to build predictive models only on the test set. In this case, one is obliged to utilize the invariance of conditional probability to adapt or transfer knowledge from the training set, termed as transfer learning \cite{pan2009survey} or domain adaptation \cite{jiang2007instance, blitzer2006domain}. Intuitively, to correct for covariate shift (i.e., cancel the bias from the training set), one can reweight the training data by assigning more weights to observations where the test data locate more often. Indeed, the key to many approaches addressing covariate shift is the estimation of importance sampling weights, or the Radon-Nikodym derivative (RND) of ${dP_{te}}/{dP_{tr}}$ between $P_{te}$ and $P_{tr}$ \cite{sugiyama2008direct,bickel2007discriminative,kanamori2012statistical,cortes2008sample,yao2010boosting,pardoe2010boosting, scholkopf2002learning,quionero2009dataset,sugiyama2012machine}. Among them is the popular kernel mean matching (KMM) \cite{huang2007correcting,quionero2009dataset}, which estimates the importance weights by matching means in a reproducing kernel Hilbert space (RKHS) and can be implemented efficiently by quadratic programming (QP).
	
	Despite the demonstrated efficiency in many covariate shift problems \cite{sugiyama2008direct, quionero2009dataset,gretton2009covariate}, KMM can suffer from high variance, due to several reasons.
	The first one regards the RKHS assumption. As pointed out in \cite{yu2012analysis}, under a more realistic assumption from learning theory \cite{cucker2007learning}, when the true regression function does not lie in the RKHS but a general range space indexed by a smoothness parameter $\theta>0$, KMM degrades to sub-canonical rate $\mathcal O (n_{tr}^{-\frac{\theta}{2\theta+4}}+n_{te}^{-\frac{\theta}{2\theta+4}})$ from the parametric rate $\mathcal O (n_{tr}^{-\frac{1}{2}}+n_{te}^{-\frac{1}{2}})$. Second, if the discrepancy between the training and testing distributions is large (e.g., test samples concentrate on regions where few training samples are located), the RND becomes unstable and leads to high resulting variance \cite{blanchet2012state}, partially due to an induced sparsity as most weights shrink towards zero while the non-zero ones surge to huge values. This is an intrinsic challenge for reweighting methods that occurs even if the RND is known in closed-form. One way to bypass it is to identify model misspecification \cite{wen2014robust}, but as mentioned in \cite{sugiyama2008directjournal}, the cross-validation for model selection needed in many related methods often requires the importance weights to cancel biases and the necessity for reweighting remains.
	

	In this paper we propose a method to reduce the variance of KMM in covariate shift problems. Our method relies on an estimated regression function and the application of the importance weighting on the \textit{residuals} of the regression. Intuitively, the residuals have smaller magnitudes than the original loss values, and the resulting reweighted estimator is thus less sensitive to the variances of weights. Then, we cancel the bias incurred by the use of residuals by a judicious compensation through the estimated regression function evaluated on the test set.
	
	Our method shares similarities with the Doubly Robust (DR) estimator in causal inference problems \cite{kennedy2017non}. However, different from DR, we do not require semi-parametric estimates of the baseline prediction (corresponding to our regression function g) and conditional probability (corresponding to our importance weight) to both converge at rates $O(n^{\alpha})$ for $\alpha>1/4$. In particular, we specialize our method by using a nonparametric regression (NR) function constructed from regularized least square in RKHS \cite{cucker2007learning,smale2007learning,sun2009note}, also known as the Tikhonov regularized learning algorithm \cite{evgeniou2000regularization}. We show that our new estimator achieves the rate $\mathcal O (n_{tr}^{-\frac{\theta}{2\theta+2}}+n_{te}^{-\frac{\theta}{2\theta+2}})$, which is superior to the best-known rate of KMM in \cite{yu2012analysis}, with the same computational complexity of KMM. Although the gap to the parametric rate is yet to be closed, the new estimator certainly seems to be a step towards the right direction.  To put into perspective, we also compare with an alternate approach in \cite{yu2012analysis} which constructs an NR function using the training set and then predicts by evaluating on the test set. Such an approach leads to a better dependence on the test size but worse dependence on the training size than KMM. Our estimator, which can be viewed as an ensemble of KMM and NR, achieves a convergence rate that is either superior or matches both of these methods, thus in a sense robust against both estimators. In fact, we show  our estimator can be motivated both from a variance reduction perspective on KMM using control variates \cite{nelson1990control,glynn2002some} and a bias reduction perspective on NR.
	


	
	Another noticable feature of the new estimator relates to data aggregation in empirical risk minimization (ERM). Specifically, when KMM is applied in learning algorithms or ERMs, the resulting optimal solution is typically a
	finite-dimensional span of the training data mapped into feature
	space \cite{scholkopf2001generalized}. The optimal solution of our estimator, on the other hand, depends on both the training and testing data, thus highlighting a different and more efficient information leveraging that utilizes both data sets simultaneously.
	
	The paper is organized as follows. Section 2 reviews the background on KMM and NR that motivates our estimator. Section 3 presents the details of our estimator and studies its convergence property. Section 4 generalizes our method to ERM. Section 5 demonstrates experimental results.
	
	\section{Background and Motivation }
		Denote $P_{tr}$ to be the probability measure for training variables $X^{tr}$ and $P_{te}$ for test variables $X^{te}$. 
	\begin{assump}\label{as1}
		$P_{tr}(dy|\bm{x})=P_{te}(dy|\bm{x})$.
	\end{assump}
	\begin{assump}\label{as2}
		The Radon-Nikodym derivative $\beta(\bm{x})\triangleq \frac{dP_{te}}{dP_{tr}} (\bm{x})$ exists and is bounded by $B<\infty$.
	\end{assump}
	\begin{assump}\label{as3}
		The covariate space $\mathcal X$ is compact and the label space $\mathcal Y \subseteq [0,1]$. Furthermore, there exists a kernel $ K(\cdot,\cdot) : \mathcal X \times \mathcal X \rightarrow \mathbb R $ which induces an RKHS $\mathcal H$ and a canonical feature map $\Phi(\cdot) : \mathcal X \rightarrow \mathcal H$ such that  $ K(\bm{x},\bm{x}')=\langle \Phi (\bm{x}), \Phi(\bm{x}') \rangle_{\mathcal H}$ and $\|\Phi (\bm{x})\|_{\mathcal H} \leq R$  for some $0<R<\infty$.
	\end{assump}
	Assumption \ref{as1} is the covariate shift assumption which states the conditional distribution $P(dy|\bm{x})$ remains invariant while the marginal $P_{tr}(\bm{x})$ and $P_{te}(\bm{x})$ differ. Assumptions \ref{as2} and \ref{as3} are common for establishing theoretical results. Specifically, Assumption \ref{as2} can be satisfied by restricting the support of $P_{te}$ and $P_{tr}$ on a compact set, although $B$ could be potentially large.
	
	\subsection{Preliminaries and Existing Approaches}
	Given $n_{tr}$ labelled training data $\{(\bm{x}_j^{tr}, \bm{y}_j^{tr})\}_{j=1}^{n_{tr}}$ and $n_{te}$  unlabelled test data $\{\bm{x}_i^{te}\}_{i=1}^{n_{te}}$ (i.e., $\{y_i^{te}\}_{i=1}^{n_{te}}$ are unavailable), the goal is to estimate $\nu=\mathbb E [Y^{te}]$. The KMM estimator \cite{huang2007correcting, gretton2009covariate} is $$V_{KMM}=\frac{1}{n_{tr}}\sum_{j=1}^{n_{tr}}\hat\beta (\bm{x}^{tr}_j) {y}_j^{tr},$$ where $\hat{\beta}(\bm{x}_j^{tr})$ are solutions of a QP that attempts to match the means of training and test sets in the feature space using weights $\hat{\bm{\beta}}$:
	\begin{equation}\label{KMMbeta}
	\begin{aligned}
	&\min_{\bm{\hat{\beta}}}   \Bigl\{\hat L(\hat{\bm{\beta}}) \triangleq \big\|\frac{1}{n_{tr}}\sum_{j=1}^{n_{tr}}\hat\beta_j\Phi(\bm{x}_j^{tr})-\frac{1}{n_{te}}\sum_{i=1}^{n_{te}} \Phi(\bm{x}_i^{te})\big\|^2_{\mathcal H} \Bigr\}  \quad \\
	&\textrm{s.t. }  0\leq \hat\beta_j\leq B, \forall  1 \leq j \leq n_{tr}.
	\end{aligned}
	\end{equation}
	Notice we write $\hat{\beta}_j$ as $\hat{\beta}(\bm{x}_j^{tr})$ in $V_{KMM}$ informally to highlight $\hat{\beta}_j$ as estimates of $\beta(\bm{x}_j^{tr})$.  The fact that (\ref{KMMbeta}) is a QP can be verified by the kernel trick, as in \cite{gretton2009covariate}. Indeed, define matrix ${K}_{ij}= K(\vx_i^{tr},\vx_j^{tr})$ and $\kappa_j\triangleq \frac{n_{tr}}{n_{te}}\sum_{i=1}^{n_{te}}K(\vx_j^{tr},\vx_i^{te})$, optimization (\ref{KMMbeta}) is equivalent to
	\begin{equation}\label{KMMbeta1}
	\begin{aligned}
	&\min_{\bm{\hat{\beta}}} \quad   \frac{1}{n^2_{tr}} \bm{\hat{\beta}}^T \mK \bm{\hat \beta}-\frac{2}{n_{tr}^2}\bm{\kappa}^T\bm{\hat\beta} ,\\
	&\textrm{s.t. }  0\leq \hat\beta_j\leq B, \forall  1 \leq j \leq n_{tr}.
	\end{aligned}
	\end{equation}
	In practice, a constraint $ \big|\frac{1}{n_{tr}}\sum_{j=1}^{n_{tr}}\hat\beta_j-1\big| \leq \epsilon$ for a tolerance $\epsilon>0$ is included to regularize the $\hat{\bm{\beta}}$ towards the RND. As in \cite{yu2012analysis}, we omit them to simplify analysis. On the other hand, the NR estimator $$V_{NR}=\frac{1}{n_{te}} \sum_{i=1}^{n_{te}} \hat g(\bm{x}_i^{te}),$$ is based on $\hat g(\cdot)$, some estimate of the regression function $g(\bm{x})\triangleq \mathbb E [Y|\bm{x}]$. Notice the conditional expectation is taken regardless of $\bm{x}\sim P_{tr}$ or $P_{te}$. Here, we consider a $
	\hat g(\cdot)$ that is estimated nonparametrically by regularized least square in RKHS:
	\begin{equation}\label{empreg}
	\hat g_{\gamma,data}(\cdot)=\operatorname*{argmin}_{f\in\mathcal H} \big\{ \frac{1}{m} \sum_{j=1}^{m} (f(\bm{x}^{tr}_j)-y^{tr}_j)^2 +\gamma \|f\|^2_{\mathcal H}\big\},
	\end{equation} 
	where $\gamma$ is a regularization term to be chosen and the subscript $data$ represents $\{ (\bm{x}^{tr}_j,y^{tr}_j)\}_{j=1}^m$. Using the representation theorem \cite{scholkopf2001generalized}, optimization problem (\ref{empreg}) can be solved in closed form with $\hat g_{\gamma, data}(\bm{x})=\sum_{j=1}^m \alpha^{reg}_j K(\bm{x}_j^{tr}, \bm{x} )$ where 
	\begin{equation}\label{empregs}
	\bm{\alpha}^{reg}=( \mK+\gamma \mI)^{-1} \vy^{tr},
	\end{equation}
	 and  $\vy^{tr}=[y_1^{tr},...,y_{m}^{tr}]$.
	\subsection{Motivation}
	Depending on properties of $g(\cdot)$, \cite{yu2012analysis} proves different rates of KMM. The most notable case is when $g\notin \mathcal H$ but rather $g(\cdot) \in Range(\mathcal T^{\frac{\theta}{2\theta+4}}_K)$, where $\mathcal T_K$ is the integral operator $(\mathcal T_K f)(x')=\int_{\mathcal X}  K(x',x)f(x)P_{tr}(dx)$ on $\mathscr L^2_{P_{tr}}$. In this case, \cite{yu2012analysis} characterize $g$ with the approximation error
	\begin{equation}\label{aperr}
	\mathcal A_2(g,F)\triangleq \inf_{\|f\|_{\mathcal H}\leq F} \|g-f\|_{\mathscr L^2_{P_{tr}}} \leq CF^{-\frac{\theta}{2}},
	\end{equation}
	and the rates of KMM drops to sub-canonical $|V_{KMM}-\nu|=\mathcal O(n_{tr}^{-\frac{\theta}{2\theta+4}}+n_{te}^{-\frac{\theta}{2\theta+4}})$, as opposed to $\mathcal O  (n_{tr}^{-\frac{1}{2}}+n_{te}^{-\frac{1}{2}})$ when $g\in\mathcal H$. As shown in Lemma 4 in the Appendix and Theorem 4.1 of \cite{cucker2007learning}), (\ref{aperr}) is almost equivalent to $g(\cdot) \in Range(\mathcal T^{\frac{\theta}{2\theta+4}}_K)$: $g(\cdot) \in Range(\mathcal T^{\frac{\theta}{2\theta+4}}_K)$ implies (\ref{aperr}) while (\ref{aperr}) leads to $g(\cdot) \in Range(\mathcal T^{\frac{\theta}{2\theta+4}-\epsilon}_K)$ for any $\epsilon >0$.  We adopt the characterization $g(\cdot) \in Range(\mathcal T^{\frac{\theta}{2\theta+4}}_K)$ as our analysis is based on related learning theory estimates.  In particular, our proofs rely on these estimates and are different from \cite{yu2012analysis}. For example, in (\ref{empreg}), $\gamma$ is used as a free parameter for controlling $\|f\|_{\mathcal H}$, whereas \cite{yu2012analysis} uses the parameter $F$ in (\ref{aperr}). Although the two approaches are equivalent from an optimization viewpoint, with $\gamma$ being the Lagrange dual variable, the former approach turns out to be more suitable to our analysis.
	
	Correspondingly, the convergence rate for $V_{NR}$ when $g(\cdot) \in Range(\mathcal T_K^{\frac{\theta}{2\theta+4}})$ is also shown in \cite{yu2012analysis} as $|V_{NR}-\nu|=\mathcal O(n_{te}^{-\frac{1}{2}}+n_{tr}^{-\frac{3\theta}{12\theta+16}})$, with $\hat g$ taken as $\hat g _{\gamma, data}$ in (\ref{empreg}) and $\gamma$ chosen optimally. The rate of $V_{KMM}$ is usually better than $V_{NR}$ due to labelling cost (i.e. $n_{tr}<n_{te}$). However, in practice the performance of $V_{KMM}$ is not always better than $V_{NR}$. This could be partially explained by the hidden dependence of  $V_{KMM}$ on potentially large $B$, but more importantly, without variance reduction, {KMM} is subject to the negative effects of unstable importance sampling weights (i.e. the $\bm{\hat\beta}$). On the other hand, the training of $\hat g$ requires labels hence can only be done on training set. Consequently, without reweighting, when estimating the test quantity $\nu$, the rate of $V_{NR}$ suffers from the bias.  
	
	This motivates the search for a robust estimator which does not require prior knowledge on the performance of $V_{KMM}$ or $V_{NR}$ and can, through a combination, reach or even surpass the best performance among both. For simplicity, we use the mean squared error (MSE) criterion $\text{MSE}(V)=\text{Var}(V)+(\text{Bias}(V))^2$ and assume an additive model $Y=g(X)+\mathcal E$ where $\mathcal E \sim \mathcal N(0,\sigma^2)$ is independent with $X$ and other errors. Under this framework, we motivate a remedy from two perspectives:
	
	\paragraph{Variance Reduction for KMM:} Consider an idealized KMM with $V_{KMM} \triangleq \frac{1}{n_{tr}}\sum_{j=1}^{n_{tr}}\beta(\bm{x}_j^{tr})y_j^{tr}$ and $\beta(\cdot)$ being the true RND. Since $$\mathbb E[\beta(X^{tr})Y^{tr}]=\mathbb E_{\bm{x}\sim P_{tr}}(\beta (\bm{x})g(\bm{x}))=\mathbb E_{\bm{x}\sim P_{te}}[g(\bm{x})]=\nu,$$ $V_{KMM}$ is unbiased  and the only source of MSE becomes the variance. It then follows from standard control variates that, given an estimator $V$ and a zero-mean random variable $W$, we can set $t^\star=\frac{\text{Cov}(V,W)}{\text{Var}(W)}$ and use $V-t^\star W$ to obtain $$\min_{t}\text{Var}(V-tW) 
=(1-\text{corr}^2(V,W)) \text{Var}(V)
\leq \text{Var}(V),$$ 
	 without altering the mean of $V$. Thus we can use $$W=\frac{1}{n_{tr}}\sum_{j=1}^{n_{tr}}\beta(\bm{x}_j^{tr})(\hat g(\bm{x}_j^{tr}))-\frac{1}{n_{te}}\sum_{i=1}^{n_{te}}\hat g(\bm{x}_i^{te})$$ with $t^\star=\frac{\text{Cov}(V_{KMM},W)}{\text{Var}(W)}$. To calculate $t^\star$, suppose $X^{te}$ and $X^{tr}$ are independent, then we have
	\begin{align*}
	\text{Cov}(V_{KMM},W)
	=&\frac{1}{n_{tr}}\text{Cov}(\beta(X^{tr})Y^{tr}, \beta(X^{tr})\hat g(X^{tr})) \nonumber\\
	=& \frac{1}{n_{tr}}\text{Cov}(\beta(X^{tr})g(X^{tr}), \beta(X^{tr})\hat g(X^{tr})) \nonumber\\
	\approx& \frac{1}{n_{tr}} \text{Var}(\beta(X^{tr})\hat{g}(X^{tr})),
	\end{align*}
	if $\hat g$ is close enough to $g$. On the other hand, in the usual case where $n_{te}\gg n_{tr}$,
	\begin{align*}
	\text{Var}(W)=&\frac{1}{n_{tr}}\text{Var}(\beta(X^{tr})\hat g(X^{tr}))+\frac{1}{n_{te}}\text{Var}(\hat g(X^{te}))\nonumber\\
	& \approx \frac{1}{n_{tr}}\text{Var}(\beta(X^{tr})\hat g(X^{tr})).
	\end{align*}
	Thus, $t^\star\approx 1$ which gives our estimator $$V_{R}=\frac{1}{n_{tr}}\sum_{j=1}^{n_{tr}}\beta(\bm{x}_j^{tr})({y}_j^{tr}-\hat g(\bm{x}_j^{tr}))+ \frac{1}{n_{te}}\sum_{i=1}^{n_{te}}\hat g(\bm{x}_i^{te}).$$
	
	\paragraph{Bias Reduction for NR:} Consider the NR estimator $V_{NR}\triangleq \frac{1}{n_{te}}\sum_{i=1}^{n_{te}}\hat g(\bm{x}_i^{te})$. Assuming again the common case where $n_{te}\gg n_{tr}$, we have $$\text{Var}(V_{NR})=\frac{1}{n_{te}} \text{Var}(\hat g(X^{te}))\approx0,$$ and the main source of MSE is bias $\mathbb E_{\bm{x}\sim P_{te}} [g(\bm{x})-\hat g(\bm{x})]$. If we add $W=\frac{1}{n_{tr}}\sum_{j=1}^{n_{tr}}\beta(\bm{x}_j^{tr})(y_j^{tr}-\hat g(\bm{x}_j^{tr}))$ to $V_{NR}$, we eliminate the bias which gives the same estimator $$V_{R}=\frac{1}{n_{tr}}\sum_{j=1}^{n_{tr}}\beta(\bm{x}_j^{tr})({y}_j^{tr}-\hat g(\bm{x}_j^{tr}))+ \frac{1}{n_{te}}\sum_{i=1}^{n_{te}}\hat g(\bm{x}_i^{te}).$$
	

	\section{Robust Estimator }

	We construct a new estimator $V_R(\rho)$ that can be shown to perform robustly against both KMM and NR estimators discussed above. In our construction, we split the training set with a proportion $ \rho\in[0,1]$, i.e., divide $\{\bm X^{tr}, \bm Y^{tr}\}_{data}\triangleq\{ (\bm{x}_j^{tr},y_j^{tr})\}_{j=1}^{n_{tr}}$ into
	\begin{equation*}
	\{\bm X^{tr}_{KMM}, \bm Y^{tr}_{KMM}\}_{data}\triangleq \{ (\bm{x}_j^{tr},y_j^{tr})\}_{j=1}^{\lfloor\rho n_{tr}\rfloor},
	\end{equation*}
	and $$\{\bm X^{tr}_{NR}, \bm Y^{tr}_{NR}\}_{data}\triangleq\{ (\bm{x}_j^{tr},y_j^{tr})\}_{j=\lfloor\rho n_{tr}\rfloor+1}^{n_{tr}},$$
	where $\{\bm X^{tr}_{KMM}, \bm X^{te}\}_{data}\triangleq \{\{\bm{x}_j^{tr}\}_{j=1}^{\lfloor\rho n_{tr}\rfloor},\{\bm{x}_i^{te}\}_{i=1}^{n_{te}}\}$ is used to solve for the weight $\hat{\bm\beta}$ in (\ref{KMMbeta}) and  $\{\bm X^{tr}_{NR}, \bm Y^{tr}_{NR}\}_{data}$ is used to train an NR function $\hat g(\cdot)=\hat g_{\gamma,data}(\cdot)$ for some $\gamma$ as in (\ref{empreg}). Finally, we define our estimator $V_{R}(\rho)$ as 
	\begin{align}\label{deff}
	V_{R}(\rho)\triangleq& \frac{1}{\lfloor\rho n_{tr}\rfloor}\sum_{j=1}^{\lfloor\rho n_{tr}\rfloor}\hat\beta(\bm{x}_j^{tr})(y_j^{tr}-\hat g(\bm{x}_j^{tr}))\nonumber\\
	&+\frac{1}{n_{te}} \sum_{i=1}^{n_{te}} \hat g(\bm{x}_i^{te}).
	\end{align} 
	
	First, we remark the parameter $\rho$ controlling the splitting of data serves mainly for theoretical considerations. In practice, the data can be used for both purposes simultaneously. Second, as mentioned, many $\hat g$ other than (\ref{empreg}) could be considered for control variate. However, aside from the availability of closed-form expression (\ref{empregs}), $\hat g_{\gamma,data}$ is connected to the learning theory estimates \cite{cucker2007learning}. Thus, for establishing a theoretical bound, we focus on $\hat g=\hat g_{\gamma,data}$ for now. 
	
	Our main result is the convergence analysis with respect to $n_{tr}$ and $n_{te}$ which rigorously justified the previous intuition. In particular, we show that $V_R$ either surpasses or achieves the better rate between $V_{KMM}$ and $V_{NR}$. In all theorems that follow, the big-$\mathcal O$ notations can be interpreted either as $1-\delta $ high probability bound or a bound on expectation. The proofs are left in the Appendix.
	\begin{theorem}\label{est1}
		Under Assumptions \ref{as1}-\ref{as3}, if we assume $g \in Range(\mathcal T_K^{\frac{\theta}{2\theta+4}})$, the convergence rate of $V_R(\rho)$ satisfies 
		\begin{equation}
		|V_R(\rho)-\nu|=\mathcal O(n_{tr}^{-\frac{\theta}{2\theta+2}}+n_{te}^{-\frac{\theta}{2\theta+2}}),\label{main rate}
		\end{equation}
		when $\hat g$ is taken to be $\hat g_{\gamma, data}$ in (\ref{deff}) with $\gamma=n^{-\frac{\theta+2}{\theta+1}}$ and $n\triangleq\min(n_{tr},n_{te})$. 
	\end{theorem}
	\begin{coro}\label{estcoro}
		Under the same setting of Theorem \ref{est1}, if we choose $\gamma=n^{-1}$, we have
		\begin{equation}
		|V_R(\rho)-\nu|=\mathcal O(n_{tr}^{-\frac{\theta}{2\theta+4}}+n_{te}^{-\frac{\theta}{2\theta+4}})\label{main rate1} 
		\end{equation}
		and if we choose $\gamma=n_{tr}^{-1}$,
		\begin{equation}
		|V_R(\rho)-\nu|=\mathcal O(n_{tr}^{-\frac{\theta}{2\theta+4}}+n_{te}^{-\frac{1}{2}}).\label{main rate2}
		\end{equation} 
	\end{coro}

	We remark several implications. First, although not achieving canonical, (\ref{main rate}) is an improvement over the best-known $\mathcal O(n_{tr}^{-\frac{\theta}{2\theta+4}}+n_{te}^{-\frac{\theta}{2\theta+4}})$ rate of $V_{KMM}$ when $g\in Range(\mathcal T^{\frac{\theta}{2\theta+4}}_{K})$, especially for small $\theta$, suggesting that $V_R$ is more suitable than $V_{KMM}$ when $g$ is irregular. Indeed, $\theta$ is a smoothness parameter that measures the regularity of $g$. When $\theta$ increases, functions in $Range (\mathcal T_K^{\frac{\theta}{2\theta+4}})$ get smoother and $Range (\mathcal T_K^{\frac{\theta_2}{2\theta_2+4}}) \subseteq Range (\mathcal T_K^{\frac{\theta_1}{2\theta_1+4}})$ for $0<\theta_1<\theta_2$, with the limiting case that $\theta\rightarrow \infty$, $\frac{\theta}{2\theta+4}\rightarrow 1/2$ and $ Range (\mathcal T_{ K}^{\frac{1}{2}}) \subseteq \mathcal H$ (i.e. $g \in \mathcal H$) for universal kernels by Mercer's theorem. 
	
	Second, as in Theorem 4 of \cite{yu2012analysis}, the optimal tuning of $\gamma$ that leads to (\ref{main rate}) depends on the unknown parameter $\theta$, which may not be adaptive in practice. However, if one simply choose $\gamma=n^{-1}$, $V_R$ still achieves a rate no worse than $V_{KMM}$ as depicted in (\ref{main rate1}). 
	
	Third, also in Theorem 4 of \cite{yu2012analysis}, the rate of $V_{NR}$ is $\mathcal O(n_{te}^{-\frac{1}{2}}+n_{tr}^{-\frac{3\theta}{12\theta+16}})$ when $g\in Range (\mathcal T_K^{\frac{\theta}{2\theta+4}})$, which is better on $n_{te}$ but not $n_{tr}$. Since usually $n_{tr}<n_{te}$, the rate of $V_{KMM}$ generally excels. Indeed, in this case the rate of $V_{NR}$ beats $V_{KMM}$ only if $\lim_{{n}\rightarrow \infty}{n^{\frac{6\theta+8}{3\theta+6}}_{te}}/{n_{tr}} \rightarrow 0$. However, if so, $V_R$ can still achieve $\mathcal O(n_{tr}^{-\frac{\theta}{2\theta+4}}+n_{te}^{-\frac{1}{2}})$ rate in (\ref{main rate2}) which is better than $V_{NR}$, by simply taking $\gamma=n_{tr}^{-1}$, i.e., regularizing the training process more when the test set is small. Moreover, as $\theta\rightarrow\infty$, our estimator $V_R$ recovers the canonical rate $n_{tr}^{-\frac{1}{2}}$ as opposed to $n_{tr}^{-\frac{1}{4}}$ in $V_{NR}$. 
	
	Thus, in summary, when $g \in Range (\mathcal T_K^{\frac{\theta}{2\theta+4}})$, our estimator $V_R$ outperforms both $V_{KMM}$ and $V_{NR}$ across the relative sizes of $n_{tr}$ and $n_{te}$. The outperformance over $V_{KMM}$ is strict when $\gamma$ is chosen dependent on $\theta$, and the performance is matched when $\gamma$ is chosen robustly without knowledge of $\theta$. 
	
	For completeness, we consider two other characterizations of $g$ discussed in \cite{yu2012analysis}:  one is $g \in \mathcal H$ and the other is $\mathcal A_\infty(g, F)\triangleq \inf_{\|f\|_{\mathcal H}\leq F} \|g-f\| \leq C (\log F)^{-s}$ for some $C,s>0$ (e.g., $g\in H^s(\mathcal X)$ with $K(\cdot,\cdot)$ being the Gaussian kernel, where $H^s$ is the Sobolev space with integer $s$). The two assumptions are, in a sense, more extreme (being optimistic or pessimistic). The next two results show that the rates of $V_R$ in these situations match the existing ones for $V_{KMM}$ (the rates for $V_{NR}$ are not discussed in \cite{yu2012analysis} under these assumptions).
	\begin{prop}\label{ideal}
		Under Assumptions \ref{as1}-\ref{as3}, if $g \in \mathcal H$, the convergence rate of $V_R(\rho)$ satisfies $|V_R(\rho)-\nu|= \mathcal O(n_{tr}^{-\frac{1}{2}}+n_{te}^{-\frac{1}{2}})$,
		when $\hat g$ is taken to be $\hat g_{\gamma, data}$ for $\gamma>0$ in (\ref{deff}).
	\end{prop}
	\begin{prop}\label{noideal}
		Under Assumptions \ref{as1}-\ref{as3}, if $\mathcal A_\infty(g, F)\triangleq \inf_{\|f\|_{\mathcal H}\leq F} \|g-f\| \leq C (\log F)^{-s}$ for some $C,s>0$, the convergence rate of $V_R(\rho)$ satisfies 
		$|V_R(\rho)-\nu|= \mathcal O\left(\log \frac{n_{tr}n_{te}}{n_{tr}+n_{te}}\right)^{-s}$, when $\hat g$ is taken to be $\hat g_{\gamma, data}$ for $\gamma>0$ in (\ref{deff}).
	\end{prop}

	\section{Empirical Risk Minimization}
	The robust estimator can handle empirical risk minimization (ERM). Given loss function $l'(x,y;\theta):\mathcal X\times \mathbb R \rightarrow \mathbb R$
	given $\theta$ in $\mathcal D$, we optimize over $$\min_{\theta\in\mathcal D} \mathbb E[l'(X^{te},Y^{te};\theta)]=\min_{\theta\in\mathcal D} \mathbb E_{\vx\sim P_{te}}[l(\vx;\theta)],$$
	where $l(\vx;\theta)\triangleq \mathbb E_{Y|\vx} [l'(\vx,Y;\theta)]$ to find $$\theta^\star\triangleq \operatorname*{argmin}_{\theta\in\mathcal D} \mathbb E_{\vx\sim P_{te}}[l(X^{te};\theta)].$$ In practice, usually a regularization term $\Omega[\theta]$ on $\theta$ is added. For example, the KMM in \cite{huang2007correcting} considers
	\begin{equation}\label{ermkmm}
	\min_{\theta\in\mathcal D }   \frac{1}{n_{tr}}\sum_{j=1}^{n_{tr}}\hat\beta (\vx^{tr}_j) l'(\vx_j^{tr}, y_j^{tr};\theta)+\lambda \Omega[\theta].
	\end{equation}
	We can carry out a similar modification for $V_R$:
	\begin{align}\label{ermcvkmm}
	\min_{\theta\in\mathcal D } \frac{1}{\lfloor\rho n_{tr}\rfloor}&\sum_{j=1}^{\lfloor\rho n_{tr}\rfloor}\hat\beta(\vx_j^{tr})(l'(\vx_j^{tr}, y_j^{tr};\theta)-\hat l(\vx_j^{tr};\theta))\nonumber\\
	&+\frac{1}{n_{te}} \sum_{i=1}^{n_{te}} \hat l(\bm{x}_i^{te};\theta)+\lambda\Omega[\theta],
	\end{align}
	with $\hat{\bm{\beta}}$ based on $\{\bm X^{tr}_{KMM}, \bm X^{te}\}$ and $\hat l(x;\theta)$ being an estimate of $l(x;\theta)$ based on $\{\bm X^{tr}_{NR}, \bm Y^{tr}_{NR}\}$. For later reference, we note that a similar modification can also be used on $V_{NR}$:
	\begin{equation}\label{ermnr}
	\min_{\theta\in\mathcal D } \frac{1}{n_{te}} \sum_{i=1}^{n_{te}} \hat l(\bm{x}_i^{te};\theta)+\lambda\Omega[\theta].
	\end{equation}
	We discuss two classical learning problems by (\ref{ermcvkmm}).
	
	\paragraph{Penalized Least Square Regression: }
	Consider a regression problem with $l'(\vx,y;\theta)=(y-\langle\theta, \Phi(\vx)\rangle_{\mathcal H})^2$, $\Omega[\theta]=\|\theta\|^2_{\mathcal H}$ and $y\in [0,1]$. We have $$l(\vx; \theta)=\mathbb E[Y^2|\vx]-2g(\vx)\langle\theta, \Phi(\vx)\rangle_{\mathcal H}+\langle\theta, \Phi(\vx)\rangle_{\mathcal H}^2,$$
	and a candidate for $\hat l(\vx, \theta)$ is to substitute $g$ with $\hat g_{\gamma, data}$. Then, (\ref{ermcvkmm}) becomes
	\begin{align*}
	\min_{\theta\in\mathcal D} &\sum_{j=1}^{\lfloor \rho n_{tr}\rfloor} -\frac{2\beta(\vx^{tr}_j)}{\lfloor \rho n_{tr}\rfloor}(y_j^{tr}-\hat{g}(\vx_j^{tr}) )\langle\theta, \Phi(\vx_j^{tr})\rangle_{\mathcal H}\nonumber\\
	&+\frac{1}{n_{te}} \sum_{i=1}^{n_{te}}(\hat g(\vx_i^{te})-\langle\theta, \Phi(\vx)\rangle_{\mathcal H})^2 +\lambda\|\theta\|^2_\mathcal H,
	\end{align*}
	by adding and removing the components not involving $\theta$. Furthermore, it simplifies to the QP: 
	\begin{align}\label{ermcvkmm_lms}
	\min_{\bm{\alpha}\in\mathbb R^{\lfloor \rho n_{tr}\rfloor+n_{te}}} &\frac{-2\vw^T_1\mK_{tot}\bm{\alpha}}{\lfloor \rho n_{tr}\rfloor}+\lambda \bm{\alpha}^T\mK_{tot} \bm{\alpha}\nonumber\\&+\frac{(\vw_2-\mK_{tot}\bm{\alpha})^T\mW_3(\vw_2-\mK_{tot}\bm{\alpha})}{n_{te}},
	\end{align}
	by the representation theorem \cite{scholkopf2001generalized}. Here $(\mK_{tot})_{ij}=K(\vx^{tot}_i, \vx^{tot}_j)$ and $\mW_3=\text{diag}(\vw_3)$ where $\vx_{i}^{tot}=\vx_{i}^{tr}$, $(w_1)_i=\beta(\vx^{tr}_i)(y_i^{tr}-\hat{g}(\vx_i^{tr}) )$, $(w_2)_i=0$, $(w_3)_i=0$ for $1 \leq i \leq \textstyle{\lfloor \rho n_{tr}\rfloor}$ and $\vx_{i}^{tot}=\vx_{i-{\lfloor \rho n_{tr}\rfloor}}^{te}$, $(w_1)_i=0$, $(w_2)_i=\hat g (\vx_{i-\lfloor \rho n_{tr}\rfloor}^{te})$, $(w_3)_i=1$ for ${\lfloor \rho n_{tr}\rfloor}+1 \leq i \leq {\lfloor \rho n_{tr}\rfloor}+n_{te}$.
	Notice (\ref{ermcvkmm_lms}) has a closed-form solution $$\hat{\bm{\alpha}}=(\mW_3\mK_{tot}+\lambda n_{te}\mI)^{-1}(\frac{n_{te}}{\lfloor \rho n_{tr}\rfloor}\vw_1+\vw_2).$$
	
	\paragraph{Penalized Logistic Regression:} Consider a binary classification problem with $y\in \{0,1\}$, $\Omega[\theta]=\|\theta\|^2_{\mathcal H}$ and $-l'(\vx,y;\theta)=y\log (\frac{1}{1+\exp{\langle\theta, \Phi(\vx)\rangle_{\mathcal H}}})+(1-y)\log (\frac{\exp{\langle\theta, \Phi(\vx)\rangle_{\mathcal H}}}{1+\exp{\langle\theta, \Phi(\vx)\rangle_{\mathcal H}}})$.
	Thus, we have $$-l(\vx;\theta)=-g(\vx)\langle\theta, \Phi(\vx)\rangle_{\mathcal H}+\log (\frac{\exp{\langle\theta, \Phi(\vx)\rangle_{\mathcal H}}}{1+\exp{\langle\theta, \Phi(\vx)\rangle_{\mathcal H}}}),$$
	and we can again substitute $g$ with $\hat g_{\gamma, data}$. Then, (\ref{ermcvkmm}) becomes
	\begin{align*}
	&\min_{\theta\in\mathcal D } \sum_{j=1}^{\lfloor \rho n_{tr}\rfloor} \frac{\beta(\vx^{tr}_j)}{\lfloor \rho n_{tr}\rfloor}(y_j^{tr}-\hat{g}(\vx_j^{tr}) )\langle\theta, \Phi(\vx_j^{tr})\rangle_{\mathcal H}\nonumber\\
	&+\frac{1}{n_{te}} \sum_{i=1}^{n_{te}}-\hat g(\vx_i^{te})\langle\theta, \Phi(\vx_i^{te})\rangle_{\mathcal H} +\lambda\|\theta\|^2_\mathcal H \nonumber\\
	&+\log (\frac{\exp{\langle\theta, \Phi(\vx_i^{te})\rangle_{\mathcal H}}}{1+\exp{\langle\theta, \Phi(\vx_i^{te})\rangle_{\mathcal H}}}).
	\end{align*}
	which again simplifies to, by \cite{scholkopf2001generalized}, the convex program:
	\begin{align}\label{ermcvkmm_log}
	\min_{\bm{\alpha}\in\mathbb R^{\lfloor \rho n_{tr}\rfloor+n_{te}}} &\frac{\vw^T_1\mK_{tot}\bm{\alpha}}{\lfloor \rho n_{tr}\rfloor}-\frac{\vw^T_2\mK_{tot}\bm{\alpha}}{n_{te}} +\lambda \bm{\alpha}^T\mK_{tot} \bm{\alpha} \nonumber\\
	&+\frac{\sum_{i=1}^{n_{te}}\log(\frac{\exp{(\mK_{tot}\bm{\alpha})_{\lfloor \rho n_{tr}\rfloor+i}}}{1+\exp{(\mK_{tot}\bm{\alpha})_{\lfloor \rho n_{tr}\rfloor+i}}})}{n_{te}}.
	\end{align}
	Both (\ref{ermcvkmm_lms}) and (\ref{ermcvkmm_log}) can be optimized efficiently by standard solvers. Notably, derived from (\ref{ermcvkmm}), an optimal solution is in the form $\hat \theta=\sum_{i=1} \hat\alpha_i K(\bm{x}_i^{tot}, \bm{x} )$ which spans on both training and test data. In contrast, the solution of (\ref{ermkmm}) or (\ref{ermnr}) only spans on one of them. For example, as shown in \cite{huang2007correcting}, the penalized least square solution for (\ref{ermkmm}) is $\hat \theta=\sum_{i=1} \hat\alpha_i K(\bm{x}_i^{tr}, \bm{x} )$ where $$\hat{\bm{\alpha}}=(\mK+ n_{te}\lambda \text{ diag}(\hat{\bm{\beta}})^{-1})^{-1} \vy^{tr}$$ (we use $\hat{\bm{\alpha}}=(\text{ diag}(\hat{\bm{\beta}})\mK+ n_{te}\lambda \mI)^{-1} \text{ diag}(\hat{\bm{\beta}})\vy^{tr}$ in experiments to avoid invertibility issues caused by the sparsity of $\hat{\bm{\beta}}$), so only the training data are in the span of the feature space that constitutes $\hat\theta$. The aggregation of both sets suggests a more effective utilization of data .
	We conclude with a theorem on ERM similar to Corollary 8.9 in \cite{gretton2009covariate}, which guarantees the convergence of the solution of (\ref{ermcvkmm}) in a simple setting.
	\begin{theorem}\label{erm}
		Assume $l(x;\theta)$ and $ \hat l(x;\theta) \in\mathcal  H$ can be expressed as $\langle\Phi(x), \theta\rangle_{\mathcal H}+f(x;\theta)$ with $||\theta||_{\mathcal H} \leq C$ and $l'(x, y ;\theta) \in\mathcal  H$  as $\langle\Upsilon(x,y), \Lambda\rangle_{\mathcal H}+f(x;\theta)$ with $||\Lambda||_{\mathcal H} \leq C$. Denote this class of loss functions $\mathcal G$ and further assume $l(x;\theta)$ are continuous, bounded by $D$ and $L$-Lipschitz on $\theta$ uniformly over $x$ for $(\theta,x)$ in a compact set $\mathcal D\times\mathcal X$. Then, the ERM with
		 \begin{align*}
		 V_R(\theta) \triangleq& \frac{1}{\lfloor\rho n_{tr}\rfloor}\sum_{j=1}^{\lfloor\rho n_{tr}\rfloor}\hat\beta(\vx_j^{tr})(l'(\vx_j^{tr}, y_j^{tr};\theta)-\hat l(\vx_j^{tr};\theta))\nonumber\\&+\frac{1}{n_{te}} \sum_{i=1}^{n_{te}} \hat l(\bm{x}_i^{te};\theta)
		 \end{align*} 
		 and $\hat\theta_{R}\triangleq \operatorname*{argmin}_{\theta\in\mathcal D} V_R(\theta)$ satisfies
		\begin{align*}
		\mathbb E[l'(X_{te},Y_{te};\hat\theta_R)]
		\leq \mathbb E[l'(X_{te},Y_{te};\theta^\star)]+\mathcal O( n_{tr}^{-\frac{1}{2}}+n_{te}^{-\frac{1}{2}}).
		\end{align*}
	\end{theorem}
	\section{Experiments}
	\subsection{Toy Dataset Regression}
	We first present a toy example to provide comparison with KMM. The data is generated as the polynomial regression example in \cite{shimodaira2000improving, huang2007correcting}, where $P_{tr}\sim\mathcal N(0.5,0.5^2)$, $P_{te}\sim\mathcal N(0, 0.3^2)$ are Gaussian distributions. The labels are generated according to $y=-x+x^3$ and observed with Gaussian noise $\mathcal N(0,0.3^2)$. We sample 500 points in both training and test data and fit a linear model using ordinary least square (OLS), KMM and our robust estimator, respectively. On the population level, the best linear fit is $y=-0.73 x$ (i.e. $\argmin_{\alpha_0,\beta_0} \mathbb E_{x\sim P_{te}}(Y- (\alpha_0 x+\beta_0))^2$ is $\alpha_0=-0.73, \beta_0=0$). For simplicity, we set the intercept $\beta_0=0$ as known and compare the fitted slopes for different estimators. We use a degree-3 polynomial kernel and set $\gamma$ in $\hat g_{\gamma,data}$ to the default value $n_{tr}^{-1}$. The tolerance $\epsilon$ for $\hat{\bm{\beta}}$ is set similarly as in \cite{huang2007correcting} with a slight tuning to avoid an overly sparse solution. The slope is fitted without regularization. In Figure 1(a), the red curve is the true polynomial regression function and the purple line is the best linear fit. The blue circle is the training data and the orange cross is the test data. For three different approaches, as well as an additional density-ratio-based method in \cite{shimodaira2000improving}, the fitted slope over 20 trials are summarized in Figure 1(b). The average value is plotted in Figure 1(a) with black (KMM), green (robust) and yellow (OLS) respectively. As we see, the robust estimator outperforms the two other methods, achieving higher accuracy than KMM and unweighted OLS and recovering the slope closest to the best one in the vast majority of trials.
	\begin{figure}[h]
		\centerline
		{\includegraphics[width=8cm]{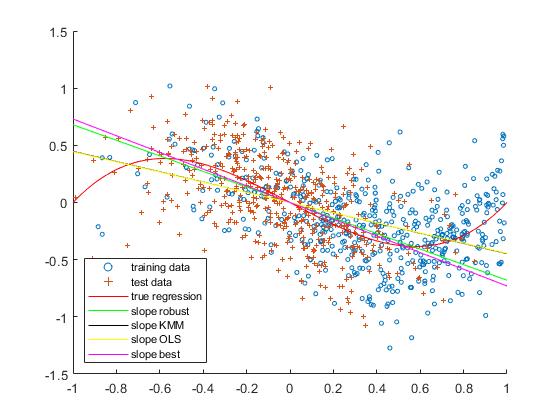} }
		\quad \text{(a)}
		
		\centerline
		{\includegraphics[height=6cm,width=8cm]{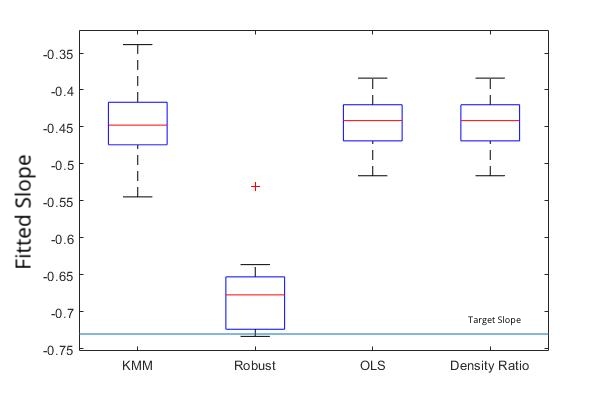} }
		\quad \text{(b)}
		\caption{ (a): Linear fit with OLS, KMM and robust estimator; (b): Boxplot on slope estimation }
	\end{figure}
\begin{figure}[h]
	\centerline
	{\includegraphics[height=4cm,width=9cm]{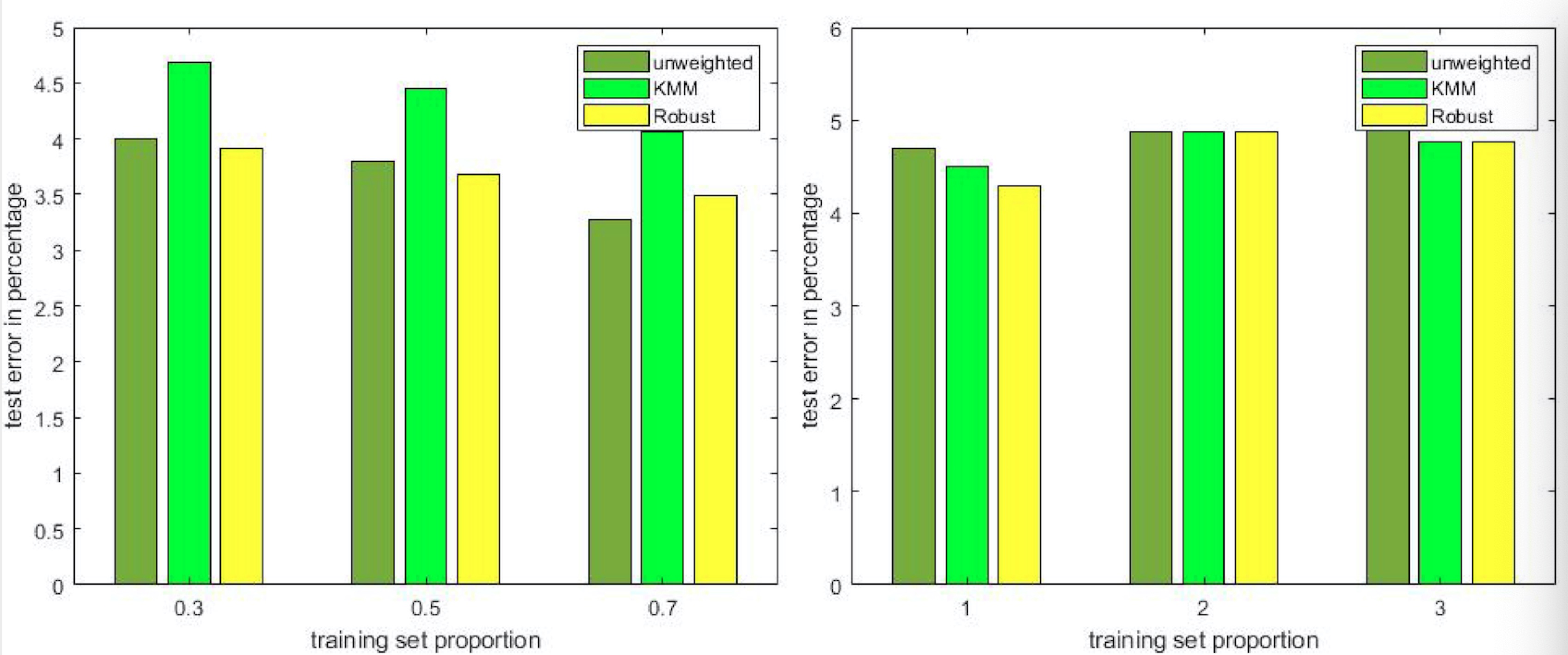} }
	\qquad\qquad\quad \text{(a)} \qquad\qquad\quad\qquad\qquad\quad\text{(b)}
	\caption{ Classification performance for (a): penalized least square regression; (b) penalized logistic regression}
\end{figure}
	
	\subsection{Real World Dataset for ERM}
	Next, we test our approach in ERM on a real world dataset, the breast cancer dataset from the UCI Archive. We consider the second biased sampling scheme in \cite{huang2007correcting} where the sampling bias operates jointly across multiple features. In particular, after randomly splitting the training and test sets based on different proportions, the training set is further subsampled with probability of selecting $\vx_i$ in the training set proportional to $\exp(-\sigma_1 \|\vx_i-\bar{\vx}\|)$ for some $\sigma_1>0$ and the training sample mean $\bar{\vx}$. Since this is a binary classification problem and we are interested in comparing different approaches, we experiment with both the penalized least square regression and the penalized logistic regression for training sets of several sizes, i.e., the proportions of the training data are 0.3, 0.5, and 0.7 respectively, with respect to the total data. We used a Gaussian kernel $\exp (-\sigma_2\|\vx_i-\vx_j\|)$ for some $\sigma_2>0$. The tolerance $\epsilon$ for $\hat{\bm{\beta}}$ is set exactly as in \cite{huang2007correcting}. For both experiments, we choose parameters $\gamma=n_{tr}^{-1}$ as default, $\lambda=5$ by cross-validation and $\sigma_1=-1/100$, $\sigma_2=\sqrt{0.5}$. Finally, we used the fitted parameters (i.e., optimal solution $\hat\theta$ in ERM) to predict the labels on the test set and compare with the hidden real ones. The summary of test error comparison is shown in Figure 2 where we use the term \textit{unweighted} to denote the case for (\ref{ermnr}), \textit{KMM} for (\ref{ermkmm}) and \textit{Robust} for (\ref{ermcvkmm}).
	The robust estimator gives the lowest test error in 5 cases out of 6 and follows KMM closely in the exceptional case, confirming our finding on its improvement over the traditional methods.

	\subsection {Simulated Dataset for Estimation} 
	To test the performace of robust estimator on an estimation problem, we simulate data from two ten-dimensional Gaussian distributions with different, randomly generated means and covariance matrices as training and test sets. The target value is $\nu=\mathbb E_{\vx\sim P_{te}}[g(\vx)]$ for an artificially constructed regression function $g(x)=\sin(c_1 \|\vx\|_2^2)+({1+\exp(\vc_2^T\vx)})^{-1}$ with random $c_1, \vc_2$ and labels are observed with Gaussian noise. The Gaussian kernel $\exp (-\sigma\|\vx_i-\vx_j\|)$ for $\sigma>0$ and a tolerance $\epsilon$ for $\hat{\bm{\beta}}$ are set with exactly the same parameters as in \cite{gretton2009covariate} with $\sigma=\sqrt{5}$, $B=1000$ and $\epsilon=\frac{\sqrt{n_{tr}}-1}{\sqrt{n_{tr}}}$. We also experiment with a different $\hat g$ by substituting $\hat g_{\gamma, data}$ for a naive linear OLS fit with a lasso regularization term $\lambda>0$. At each iteration, we use the sample mean from $10^6$ data points (without adding noise) as the true mean and calculate the average MSE over 100 estimations for $V_R$, $V_{KMM}$ and $V_{NR}$ respectively. As shown in Table 1, the performances of $V_R$ are again consistently on par with the best case scenarios, even when the form of $\hat g_{\gamma,data}$ is replaced with a naive OLS fit, suggesting the robust estimator still works well under other forms of control variate functions. Moreover, we see that the robust estimator exhibits satisfactory performance even when the usual assumption $n_{tr}<n_{te}$ is violated.
	
	\begin{table}[htb]
		\begin{center}
			\caption{Average MSE for Estimation}
			\begin{tabular}{cccc} 
				\textbf{Hyperparameters} &  &\textbf{MSE}  &\\
				($\lambda, n_{tr},n_{te}$) &  $V_{NR}$ &   $V_{KMM}$ & $V_{R}$\\ 
				\hline
					\\ [-1em]
				$(0.1, 50, 500)$ & 0.9970 & 0.9489 & \textbf{0.9134} \\
				\\ [-1em]
				
				\\ [-1em]
				$(0.1, 500, 500)$  & 1.0006 & \textbf{0.9294} & 0.9340 \\
				\\ [-1em]
				
				\\ [-1em]
				$(0.1, 500, 50)$ & 1.0021 & 0.9245 & \textbf{0.9242} \\
				\\ [-1em]
				
				\\ [-1em]
				$(10, 50, 500)$ & 0.9962 & 0.9493 & \textbf{0.9467} \\
				\\ [-1em]
				
				\\ [-1em]
				$(10, 500, 500)$  & 0.9964 & 0.9294 & \textbf{0.9288} \\

				\\ [-1em]
				$(10, 500, 50)$ & 0.9965 & \textbf{0.9245} & 0.9293 \\
				
				\\ [-1em]
				
			\end{tabular}
		\end{center}
	\end{table}

			\section{Conclusion}
			Motivated from variance and bias reduction, we introduced a new robust estimator for covariate shift problems which leads to improved accuracy over both KMM and NR in different settings. From a practical standpoint, the control variates and data aggregation enable the estimation/training process to be more stable and data-efficient at no expense of significant computational complexity increase. From an analytical standpoint, when the regression function lies in range spaces outside of RKHS, a promising progress is made to improve upon the well-known rate gap of KMM towards the parametric. For future work, note the canonical rate is still not achieved and it remains unclear the suitable tools for further improvement, if possible at all. Moreover, outside the KMM context with the regularized empirical regression function in RKHS, establishing the eligibility and effectiveness of other reweighting method coupled with different regression functions from learning schemes requires rigorous analysis.

			\newpage
			\textbf{Acknowledgements}\\
			
			We gratefully acknowledge support from the National Science Foundation under grants IIS-1849280 and CMMI-1653339/1834710.
			\bibliography{Total_1}
			\bibliographystyle{apalike}
			
			\newpage
			\onecolumn
			
				\setcounter{section}{6}
			\setcounter{coro}{1}
			\setcounter{prop}{2}
			\section{Appendix}
			Throughout the proofs, 
			$h(\cdot) \in \mathcal H$ is assumed to be an unspecified function in the RKHS. Also,
			we use $\mathbb E_X[\cdot]$ to denote expectation over the randomness of $X$ while fixing others and $\mathbb E_{|X}[\cdot] $ as the conditional expectation $\mathbb E[\cdot|X]$. Moreover we remark that
			all results involving $\hat g_{\gamma,data}$ can be interpreted either as a high probability bound or a bound on expectation over $\mathbb E_{data}$ (i.e., if we train $\hat g_{\gamma,\bm{X}^{tr}_{NR},\bm{Y}^{tr}_{NR}}$ using ${\bm{X}^{tr}_{NR},\bm{Y}^{tr}_{NR}}$, then $\mathbb E_{data}$ means $\mathbb E_{\bm{X}^{tr}_{NR},\bm{Y}^{tr}_{NR}}$ ). The same interpretation applies for the results with Big-$\mathcal O$ notations.  Finally, constants $C_2,C'_2$, $C_3$, $C'_3$ and $C''_3$ as well as similar constants introduced later which depend on $R, g(\cdot)$ or $\delta$ (for $1-\delta$ high probability bound) will sometimes be denoted by a common $C$ during the proofs for ease of presentation.
			
			\subsection{Preliminaries}
			\begin{lemma}\label{suptoH}
				Under Assumption 3, for any $f\in\mathcal H$, we have 
				\begin{equation}
				\| f \|_{\infty}=\sup_{x\in\mathcal X }|\langle f(\cdot), \Phi(\cdot,x)\rangle_{\mathcal H}|\leq R\|f\|_{\mathcal H}.
				\end{equation}
				and consequently $\|f\|_{\mathscr L^2_{P_{tr}}} \leq R\|f\|_{\mathcal H}$ as well.
				
			\end{lemma}

			\begin{lemma}[Azuma-Hoeffding]\label{azulema}
				Let $X_1,...,X_n$ be independent and identically distributed random variables with $0\leq X\leq B$, then
				\begin{equation}
				P(|\frac{1}{n}\sum_{i=1}^n \bm{x}_i-\mathbb E[X]|>\epsilon) \leq 2e^{-\frac{2n\epsilon^2}{B^2}}.
				\end{equation}
			\end{lemma}
			
			\begin{coro}\label{azucoro}
				Under the same assumption of Lemma \ref{azulema}, with probability at least $1-\delta$,
				\begin{equation}
				|\frac{1}{n}\sum_{i=1}^n \bm{x}_i-\mathbb E[X]| \leq B\sqrt{\frac{1}{2n} \log \frac{2}{\delta}}.
				\end{equation}
			\end{coro}

			Moreover, an important $(1-\delta)$-probability bound we shall use later for $\hat L (\bm{\beta}_{|\vx_1^{tr},...,\vx^{tr}_{n_{tr}}}))$ follows from  \cite{yu2012analysis} (see also \cite{gretton2009covariate} and \cite{pinelis1994optimum}):
			\begin{align}\label{term3'}
			\hat L (\bm{\beta}_{|\bm{x}_1^{tr},...,\bm{x}^{tr}_{n_{tr}}}))=&\bigg\| \frac{1}{n_{tr}}\sum_{j=1}^{n_{tr}}\beta(\bm{x}_j^{tr})\Phi(\bm{x}_j^{tr}) - \frac{1}{n_{te}} \sum_{i=1}^{n_{te}} \Phi(\bm{x}_i^{te}) \bigg\|_{\mathcal H} \nonumber\\
			\leq &\sqrt {2\log\frac{2}{\delta}}R\sqrt{\bigg(\frac{B^2}{n_{tr}}+\frac{1}{n_{te}}\bigg)}.
			\end{align}

			\subsection{Learning Theory Estimates}
			
			To adopt the more realistic assumption as in \cite{yu2012analysis,cucker2007learning} that the true regression function $g(\cdot) \notin \mathcal H$ but rather $g(\cdot) \in Range(\mathcal T^{\frac{\theta}{2\theta+4}}_{ K})$, we need results from learning theory.

			First, define $\zeta\triangleq \frac{\theta}{2\theta+4}$ for some $\theta>0$ so that $0<\zeta <1/2$. Given $g(\cdot) \in Range(\mathcal T^{\zeta}_{ K})$ and $m$ training sample $ \{ (\bm{x}_j,y_j)\}_{j=1}^{m}$ (sampled from $P_{tr})$), we define $g_{\gamma}(\cdot) \in \mathcal H:\mathcal X\rightarrow \mathbb R$ to be 
			\begin{equation}\label{regreg}
			g_{\gamma}(\cdot)=\operatorname*{argmin}_{f\in\mathcal H} \bigg\{  \|f-g\|^2_{\mathscr L^2_{P_{tr}}} +\gamma \|f\|^2_{\mathcal H}\bigg\}
			\end{equation} 
			where $\|f-g\|_{\mathscr L^2_{P_{tr}}}=\sqrt{\mathbb E_{\bm{x}\sim P_{tr}}(f(\bm{x})-g(\bm{x}))^2}$ denotes the $\mathscr L^2$ norm under $P_{tr}$. On the other hand, $\hat g_{\gamma,data}(\cdot) \in \mathcal H$ is defined in (3)
			\begin{equation*}
			\hat g_{\gamma,data}(\cdot)=\operatorname*{argmin}_{f\in\mathcal H} \bigg\{ \frac{1}{m} \sum_{j=1}^{m} (f(\bm{x}_j)-y_j)^2 +\gamma \|f\|^2_{\mathcal H}\bigg\}.
			\end{equation*} 
			
			Moreover, following the notations in Section 4.5 of \cite{cucker2007learning}, given Banach space $(\mathscr L^2_{P_{tr}},\|\cdot\|_{\mathscr L^2_{P_{tr}}})$ and our kernel-induced Hilbert subspace $(\mathcal H,\|\cdot\|_{\mathcal H})$,  we define a $\tilde {\mathbb K}$-functional: $\mathscr L^2_{P_{tr}} \times (0,\infty) \rightarrow \mathbb R$ to be 
			\begin{equation*}
			\tilde{\mathbb K} (l,\gamma) \triangleq \inf_{f\in\mathcal H} \{\|l-f\|_{\mathscr L^2_{P_{tr}}}+\gamma\|f\|_{\mathcal H}\}
			\end{equation*}
			for $l(\cdot) \in \mathscr L^2_{P_{tr}}$ and $t>0$.
			For $0<r<1$, the interpolation space $(\mathscr L^2_{P_{tr}},\mathcal H)_r$ consists of all the elements $l(\cdot) \in \mathscr L^2_{P_{tr}}$ such that 
			\begin{equation}\label{kfinite}
			\|l\|_r \triangleq \sup_{\gamma>0}\frac{\tilde{\mathbb K}(l,\gamma)}{\gamma^r}<\infty.
			\end{equation}
			\begin{lemma}\label{2to1b}
				Define $\mathbb K :\mathscr L^2_{P_{tr}} \times (0,\infty) \rightarrow \mathbb R$ to be
				\begin{equation}
				\mathbb K(l,\gamma)\triangleq \inf_{f\in\mathcal H} \{\|l-f\|_{\mathscr L^2_{P_{tr}}}^2+\gamma \|f\|^2_{\mathcal H}\}.
				\end{equation}
				Then for any $l(\cdot)\in (\mathscr L^2_{P_{tr}},\mathcal H)_r$, we have
				\begin{equation}
				\sup_{\gamma>0}\frac{\mathbb K(l,\gamma)}{\gamma^r} \leq \bigg(\sup_{\gamma>0} \frac{\tilde{\mathbb K} (l,\sqrt{\gamma})}{(\sqrt{\gamma})^r}\bigg)^2= \|l\|_r^2<\infty.
				\end{equation}
			\end{lemma}
			\begin{proof}
				It follows from $\sqrt{a+b} \leq \sqrt{a}+\sqrt{b}, \quad \forall a,b \geq 0$ that
				\begin{equation}
				\sqrt{\mathbb K(l,\gamma)} \leq \tilde{\mathbb K}(l,\sqrt{\gamma}).
				\end{equation}
				Thus, for any $l(\cdot)\in (\mathscr L^2_{P_{tr}},\mathcal H)_r$, we have
				\begin{equation}\label{2boundfinite}
				\sup_{\gamma>0}\frac{\mathbb K(l,\gamma)}{\gamma^r} \leq \bigg(\sup_{\gamma>0} \frac{\tilde{\mathbb K} (l,\sqrt{\gamma})}{(\sqrt{\gamma})^r}\bigg)^2= \|l\|_r^2<\infty.
				\end{equation}
			\end{proof}

			On the other hand, assuming $g(\cdot)\in Range(\mathcal T_K^{\frac{\theta}{2\theta+4}})$, it follows from the proof of Theorem 4.1 in \cite{cucker2007learning} that
			\begin{equation}
			g(\cdot)\in (\mathscr L_{P_{tr}}^2, \mathcal H^{+})_{\frac{\theta}{\theta+2}}
			\end{equation}
			where $\mathcal H^{+}$ is a closed subspace of $\mathcal H$ spanned by eigenfunctions of the kernel $K$ (e.g., $\mathcal H^{+}=\mathcal H$ when $P_{tr}$ is non-degenerate, see Remark 4.18 of \cite{cucker2007learning}). Indeed, the next lemma shows we can measure smoothness through interpolation space just as range space.
			\begin{lemma}\label{g_belong}
				Assuming $P_{tr}$ is non-degenerate on $\mathcal X$. Then if $g\in Range(\mathcal T_K^{\frac{\theta}{2\theta+4}})$, we have $g\in (\mathscr L^2_{P_{tr}}, \mathcal H)_{\frac{\theta}{\theta+2}}$. On the other hand, if $g\in (\mathscr L^2_{P_{tr}}, \mathcal H)_{\frac{\theta}{\theta+2}}$, then $g\in Range(\mathcal T_K^{\frac{\theta}{2\theta+4}-\epsilon})$ for all $\epsilon>0$.
			\end{lemma}
			\begin{proof}
				The proof follows from Theorem 4.1, Corollary 4.17 and Remark 4.18 of \cite{cucker2007learning}.
			\end{proof}

			Now we are ready to adopt some common assumptions and theoretical results from  learning theory in RKHS. They can be found in \cite{cucker2007learning,sun2009note,smale2007learning,yu2012analysis}.  First, given $g(\cdot) \in Range(\mathcal T^{\zeta}_{ K})$ and $m$ training sample $ \{ (\bm{x}_j,y_j)\}_{j=1}^{m}$ (sampled from $P_{tr})$), it follows from Lemma 3 of \cite{smale2007learning} (see as well Remark 3.3 and Corollary 3.2 in \cite{sun2009note}) that
			\begin{equation}\label{learn1}
			\|g_{\gamma}-g\|_{\mathscr L^2_{P_{tr}}} \leq C_2 \gamma^{\zeta}.
			\end{equation}
			Second, it follows from Theorem 3.1 in \cite{sun2009note} as well as \cite{smale2007learning,sun2010regularized} that
			\begin{equation}\label{learn1.5}
			\|g_{\gamma}-\hat g_{\gamma,data}\|_{\mathscr L^2_{P_{tr}}} \leq C_2'(\gamma^{-1/2}m^{-1/2}+\gamma^{-1}m^{-3/4}),
			\end{equation}
			and, by the triangle inequality,
			\begin{equation}\label{learn3}
			\mathbb\|g-\hat g_{\gamma,data}\|_{\mathscr L^2_{P_{tr}}} \leq C_3 (\gamma^{\zeta}+\gamma^{-1/2}m^{-1/2}+\gamma^{-1}m^{-3/4}).
			\end{equation}
			Notice here that by choosing $\gamma=m^{-\frac{3}{4(1+\zeta)}}$, we recover  Corollary 3.2 of \cite{sun2009note}. Finally it follows from Theorem 1 of \cite{smale2007learning}, we have
			\begin{equation}\label{learn2}
			\|g_{\gamma}-\hat g_{\gamma,data}\|_{\mathcal H} \leq C'_3 \gamma^{-1} m^{-1/2},
			\end{equation}
			with $C'_3=6 R \log \frac{2}{\delta}$. In fact, if we define $\sigma^2\triangleq \mathbb E_{\bm{x}\sim P_{tr}}\mathbb E_{Y|\bm{x}}(g(\bm{x})-Y)^2$, then Theorem 3 of \cite{smale2007learning} stated that
			\begin{equation}\label{learn4}
			\|g_{\gamma}-\hat g_{\gamma,data}\|_{\mathcal H} \leq C''_3 ((\sqrt {\sigma^2}+\|g_\gamma-g\|_{\mathscr L^2_{P_{tr}}})\gamma^{-1}m^{-1/2}+\gamma^{-1}m^{-1}).
			\end{equation}
			
			\subsection {Main Proofs}
			\begin{proof}[Proof of Theorem 1 and Corollary 1]
				
				If $g \in Range (\mathcal T_K^{\frac{\theta}{2\theta+4}})$ (i.e. $\zeta=\frac{\theta}{2\theta+4}$) and we set $h(\cdot)= g_{\gamma}(\cdot)$ and $\hat g=\hat g_{\gamma,\bm{X}_{NR}^{tr},\bm{Y}_{NR}^{tr}}$ for some $\gamma>0$, then
				\begin{align}\label{VR}
				&V_R(\rho)-\nu \nonumber\\
				=&\frac{1}{\lfloor\rho n_{tr}\rfloor}\sum_{j=1}^{\lfloor\rho n_{tr}\rfloor}\hat\beta(\bm{x}_j^{tr})(y_j^{tr}- g(\bm{x}_j^{tr})) +\frac{1}{\lfloor\rho n_{tr}\rfloor}\sum_{j=1}^{\lfloor\rho n_{tr}\rfloor}(\hat\beta(\bm{x}_j^{tr})-\beta(\bm{x}_j^{tr}))(g(\bm{x}_j^{tr})-h(\bm{x}_j^{tr})) \nonumber\\ 
				&+\frac{1}{\lfloor\rho n_{tr}\rfloor}\sum_{j=1}^{\lfloor\rho n_{tr}\rfloor}(\hat\beta(\bm{x}_j^{tr})-\beta(\bm{x}_j^{tr}))(h(\bm{x}_j^{tr})-\hat g(\bm{x}_j^{tr})) \nonumber\\
				&+\frac{1}{\lfloor\rho n_{tr}\rfloor}\sum_{j=1}^{\lfloor\rho n_{tr}\rfloor}\beta(\bm{x}_j^{tr})(g(\bm{x}_j^{tr})-\hat g(\bm{x}_j^{tr})) +\frac{1}{n_{te}} \sum_{i=1}^{n_{te}} \hat g(\bm{x}_i^{te})-\nu.
				\end{align}
				To bound terms in (\ref{VR}), we first use Corollary \ref{azucoro} to conclude that with probability at least $1-\delta$,
				\begin{align}\label{newterm1}
				|\frac{1}{\lfloor\rho n_{tr}\rfloor}\sum_{j=1}^{\lfloor\rho n_{tr}\rfloor} \hat \beta(\bm{x}_j^{tr})(y_j^{tr}-g(\bm{x}_j^{tr})) | \leq& B\sqrt{\frac{1}{\lfloor\rho n_{tr}\rfloor}\log \frac{2}{\delta}}=\mathcal O(n_{tr}^{-1/2}).
				\end{align}
				We hold on our discussion for the second term. For the third term, since $h,\hat g\in \mathcal H$, 
				\begin{align}\label{newterm2'}
				&\bigg|\frac{1}{\lfloor\rho n_{tr}\rfloor}\sum_{j=1}^{\lfloor\rho n_{tr}\rfloor}(\hat\beta(\bm{x}_j^{tr})-\beta(\bm{x}_j^{tr}))(h(\bm{x}_j^{tr})-\hat g(\bm{x}_j^{tr}))\bigg| \nonumber\\
				=& \bigg|\frac{1}{\lfloor\rho n_{tr}\rfloor}\sum_{j=1}^{\lfloor\rho n_{tr}\rfloor}(\hat\beta(\bm{x}_j^{tr})-\beta(\bm{x}_j^{tr}))\big\langle h-\hat g, \Phi(\bm{x}_j^{tr})\big\rangle_{\mathcal H}\bigg| \nonumber\\
				=& \bigg|\bigg\langle h-\hat g, \frac{1}{\lfloor\rho n_{tr}\rfloor}\sum_{j=1}^{\lfloor\rho n_{tr}\rfloor}(\hat\beta(\bm{x}_j^{tr})-\beta(\bm{x}_j^{tr})) \Phi(\bm{x}_j^{tr})\bigg\rangle_{\mathcal H}\bigg| \nonumber\\
				\leq & \|h-\hat g\|_{\mathcal H} (\hat L(\bm{\hat\beta})+\hat L (\bm{\beta}_{|\bm{x}_1^{tr},...,\bm{x}^{tr}_{\lfloor\rho n_{tr}\rfloor}})) \leq 2\| h-\hat g\|_{\mathcal H}\hat L (\bm{\beta}_{|\bm{x}_1^{tr},...,\bm{x}^{tr}_{\lfloor\rho n_{tr}\rfloor}}),
				\end{align}
				by definition of (1). Thus, when taking $h=g_\gamma$ and $\hat g=\hat g_{\gamma,\bm{X}_{NR}^{tr},\bm{Y}_{NR}^{tr}}$ for some $\gamma$, we can combine (\ref{term3'}) and (\ref{learn2}) to guarantee, with probability $1-2\delta$,
				\begin{align}\label{newterm2}
				&\bigg|\frac{1}{\lfloor\rho n_{tr}\rfloor}\sum_{j=1}^{\lfloor\rho n_{tr}\rfloor}(\hat\beta(\bm{x}_j^{tr})-\beta(\bm{x}_j^{tr}))(h(\bm{x}_j^{tr})-\hat g(\bm{x}_j^{tr}))\bigg| \nonumber\\
				\leq &\sqrt{8\log\frac{2}{\delta}}RC(1-\rho)^{-1/2}(\gamma^{-1}n^{-1/2}_{tr})\cdot \sqrt{\bigg(\frac{B^2}{n_{tr}}+\frac{1}{n_{te}}\bigg)}\nonumber\\
				=&\mathcal O(\gamma^{-1}n_{tr}^{-1/2}(n_{tr}^{-1}+n_{te}^{-1})^{\frac{1}{2}}).
				\end{align}
				
				For the last term $\tau\triangleq\frac{1}{\lfloor\rho n_{tr}\rfloor}\sum_{j=1}^{\lfloor\rho n_{tr}\rfloor}\beta(\bm{x}_j^{tr})(g(\bm{x}_j^{tr})-\hat g(\bm{x}_j^{tr})) +\frac{1}{n_{te}} \sum_{i=1}^{n_{te}} \hat g(\bm{x}_i^{te})-\nu$, the analysis relies the splitting of data, as we notice that 
				\begin{align}\label{zeromean}
				&\mathbb E_{|\bm X^{tr}_{NR}, \bm Y^{tr}_{NR}}\bigg[\frac{1}{\lfloor\rho n_{tr}\rfloor}\sum_{j=1}^{\lfloor\rho n_{tr}\rfloor}\beta(\bm{x}_j^{tr})(g(\bm{x}_j^{tr})-\hat g(\bm{x}_j^{tr})) +\frac{1}{n_{te}} \sum_{i=1}^{n_{te}} \hat g(X_i^{te})-\nu \bigg] \nonumber\\
				=& \mathbb E_{\bm{x}\sim P_{tr}}[\beta (\bm{x}) g(\bm{x})]-\nu- \mathbb E_{\bm{x}\sim P_{tr}}[\beta (\bm{x}) \hat g(\bm{x})] + \mathbb E_{\bm{x}\sim P_{te}}[ \hat g(\bm{x})] \nonumber\\ =&\mathbb E_{\bm{x}\sim P_{te}}[g(\bm{x})]-\nu-\mathbb E_{\bm{x}\sim P_{te}}[\hat g (\bm{x})]+\mathbb E_{\bm{x}\sim P_{te}}[\hat g (\bm{x})]\nonumber\\
				=&0.
				\end{align}
				Notice the second line follows since $\hat g(\cdot)$ is determined by $\{\bm X^{tr}_{NR}, \bm Y^{tr}_{NR}\}$ and thus is independent of $\{\bm X^{tr}_{KMM}, \bm Y^{tr}_{KMM}\}$ or $\{\bm X^{te}\}$. Thus, we have
				\begin{align}
				\text{Var} (\tau)=&\text{Var} (\mathbb E_{|\bm X^{tr}_{NR}, \bm Y^{tr}_{NR}}(\tau))+\mathbb E[ \text{Var} _{|\bm X^{tr}_{NR}, \bm Y^{tr}_{NR}}(\tau)] \nonumber\\
				=&\mathbb E [\text{Var} _{|\bm X^{tr}_{NR}, \bm Y^{tr}_{NR}}(\tau)] \nonumber\\
				=& \frac{1}{\lfloor \rho n_{tr}\rfloor} \mathbb E[\text{Var} _{\bm{x}\sim P_{tr}|\bm X^{tr}_{NR}, \bm Y^{tr}_{NR}}(\beta(\bm{x})(g(\bm{x})-\hat g(\bm{x})))]+\frac{1}{n_{te}}\mathbb E [\text{Var} _{\bm{x}\sim P_{te}|\bm X^{tr}_{NR}, \bm Y^{tr}_{NR}}(\hat g(\bm{x}))] \nonumber\\
				\leq & \frac{B^2}{\lfloor \rho n_{tr}\rfloor} \mathbb E_{\bm X^{tr}_{NR}, \bm Y^{tr}_{NR}} \|g-\hat g\|^2_{\mathscr L^2_{P_{tr}}} +\frac{1}{n_{te}} \mathbb E_{\bm X^{tr}_{NR}, \bm Y^{tr}_{NR}} \|\hat g\|^2_{\mathscr L^2_{ P_{te}}} \nonumber\\
				\leq & \frac{B^2}{\lfloor \rho n_{tr}\rfloor} \mathbb E_{\bm X^{tr}_{NR}, \bm Y^{tr}_{NR}} \|g-\hat g\|^2_{\mathscr L^2_{P_{tr}}} +\frac{B}{n_{te}} \mathbb E_{\bm X^{tr}_{NR}, \bm Y^{tr}_{NR}} \|\hat g\|^2_{\mathscr L^2_{ P_{tr}}},
				\end{align}
				and we can use the Chebyshev inequality and Lemma \ref{suptoH} to conclude, with probability at least $1-\delta$, 
				\begin{equation}\label{newterm3'}
				|\tau| \leq \sqrt{\frac{1}{\delta}}\sqrt{\frac{B^2}{\lfloor \rho n_{tr}\rfloor} \mathbb E_{\bm X^{tr}_{NR}, \bm Y^{tr}_{NR}} \|g-\hat g\|^2_{\mathscr L^2_{P_{tr}}} +\frac{BR^2}{n_{te}} },
				\end{equation}
				which becomes, by (\ref{learn3}), with probability $1-2\delta$,
				\begin{align}\label{newterm3}
				|\tau| \leq &\sqrt{\frac{1}{\delta}}\sqrt{\frac{B^2}{\lfloor \rho n_{tr}\rfloor} C(1-\rho)^{-3/4}(\gamma^{\zeta}+\gamma^{-1/2}n^{-1/2}_{tr}+\gamma^{-1}n^{-3/4}_{tr}) +\frac{BR^2}{n_{te}} }\nonumber\\
				=&\mathcal O((\gamma^{\zeta}+\gamma^{-1/2}n^{-1/2}_{tr}+\gamma^{-1}n^{-3/4}_{tr})n_{tr}^{-1/2}+n_{te}^{-1/2})
				\end{align}
				with $\zeta=\frac{\theta}{2\theta+4}$. 
				Now, to bound the second term $\frac{1}{\lfloor\rho n_{tr}\rfloor}\sum_{j=1}^{\lfloor\rho n_{tr}\rfloor}(\hat\beta(\bm{x}_j^{tr})-\beta(\bm{x}_j^{tr}))(g(\bm{x}_j^{tr})-h(\bm{x}_j^{tr}))$, we have

				\begin{align}\label{newterm4}
				&\frac{1}{\lfloor\rho n_{tr}\rfloor}\sum_{j=1}^{\lfloor\rho n_{tr}\rfloor}|(\hat\beta(\bm{x}_j^{tr})-\beta(\bm{x}_j^{tr}))(g(\bm{x}_j^{tr})- g_\gamma(\bm{x}_j^{tr}))| \nonumber\\
				\leq & \frac{B}{\lfloor\rho n_{tr}\rfloor}\sum_{j=1}^{\lfloor\rho n_{tr}\rfloor}|g(\bm{x}_j^{tr})- g_\gamma(\bm{x}_j^{tr})| \nonumber\\
				\leq & \big|\frac{B}{\lfloor\rho n_{tr}\rfloor}\sum_{j=1}^{\lfloor\rho n_{tr}\rfloor}|g(\bm{x}_j^{tr})- g_\gamma(\bm{x}_j^{tr})|-B\mathbb \|g- g_\gamma\|_{\mathscr L^1_{P_{tr}}}\big|+B\mathbb \|g- g_\gamma\|_{\mathscr L^1_{P_{tr}}} \nonumber\\
				\leq &\sqrt{\frac{1}{\delta}}\sqrt{\frac{B^2}{\rho n_{tr}}\|g-g_\gamma\|^2_{\mathscr L^2_{P_{tr}}}}+B\|g-g_\gamma\|_{\mathscr L^2_{P_{tr}}}\nonumber\\
				\leq &\sqrt{\frac{1}{\delta}}BC\gamma^{\zeta}\sqrt{\frac{1}{\rho n_{tr}}}+C\gamma^{\zeta}=\mathcal O(\gamma^\zeta)=\mathcal O(\gamma^{\frac{\theta}{2\theta+4}}).
				\end{align}
				where $\mathscr L^1_{P_{tr}}$ denotes the 1-norm $\mathbb E_{\bm{x}\sim P_{tr}}|g(\bm{x})-g_\gamma(\bm{x})|$. Notice the second-to-last line follows from the Chebyshev inequality, the Cauchy-Schwarz inequality, and the last line from (\ref{learn1}).

				Thus, when taking $h=g_{\gamma}$ and $\hat g=\hat g_{\gamma,\bm X^{tr}_{NR}, \bm Y^{tr}_{NR}}$ for some $\gamma>0$, we can combine (\ref{newterm1}), (\ref{newterm2}), (\ref{newterm3}) and (\ref{newterm4}) to have
				\begin{align}\label{final20}
				|V_{R}(\rho)-\nu| 
				= &\mathcal O(n_{tr}^{-\frac{1}{2}})+\mathcal O(\gamma^{\frac{\theta}{2\theta+4}})+\mathcal O(\gamma^{-1}n_{tr}^{-1/2}(n_{tr}^{-1}+n_{te}^{-1})^{\frac{1}{2}}) \nonumber\\
				&+\mathcal O((\gamma^{\frac{\theta}{2\theta+4}}+\gamma^{-1/2}n^{-1/2}_{tr}+\gamma^{-1}n^{-3/4}_{tr})n_{tr}^{-1/2}+n_{te}^{-1/2})\nonumber\\
				=& \mathcal O(n_{tr}^{-\frac{1}{2}}+n_{te}^{-\frac{1}{2}}+\gamma^{\frac{\theta}{2\theta+4}}+\gamma^{-\frac{1}{2}}n_{tr}^{-1}+\gamma^{-\frac{1}{2}}n_{tr}^{-\frac{1}{2}}n_{te}^{-\frac{1}{2}}),
				\end{align}
				after simplification. 
				Now, if we take $\gamma=n^{-\frac{\theta+2}{\theta+1}}$ where $n\triangleq \min(n_{tr},n_{te})$, then (\ref{final20}) becomes 
				\begin{align}\label{final200}
				&|V_{R}(\rho)-\nu|\nonumber\\
				=&\mathcal O(n^{-\frac{1}{2}}+n^{-\frac{\theta}{2(\theta+1)}}+n^{\frac{\theta+2}{2(\theta+1)}}n^{-1})=\mathcal O(n^{-\frac{\theta}{2\theta+2}})=\mathcal O(n_{tr}^{-\frac{\theta}{(2\theta+2)}}+n_{te}^{-\frac{\theta}{(2\theta+2)}}),
				\end{align}
				
				which is the statement of the theorem. However, note that if we choose $\gamma=n^{-1}$, we would achieve the convergence rate of $V_{KMM}$ as $\mathcal O(n_{tr}^{-\frac{\theta}{(2\theta+4)}}+n_{te}^{-\frac{\theta}{(2\theta+4)}})$. Moreover if $\lim_{{n}\rightarrow \infty}{n^{\frac{6\theta+8}{3\theta+6}}_{te}}/{n_{tr}} \rightarrow 0$ and we choose $\gamma=n_{tr}^{-1}$, then the rate becomes $\mathcal O(n_{tr}^{-\frac{\theta}{2\theta+4}}+n_{te}^{-\frac{1}{2}})$.
			\end{proof}

			\begin{proof}[Proof of Proposition 1]
				Fixing $\gamma>0$, if $g\in\mathcal H $ $(i.e., g\in Range (\mathcal T _{K}^{\frac{\theta}{2\theta+4}}) \text{ with } \theta\rightarrow\infty)$, then by definition of $g_\gamma$ we would have
				\begin{align}\label{hbound}
				\|g_{\gamma}\|^2_{\mathcal H} \leq& \frac{ \|g_\gamma-g\|_{\mathscr L^2_{P_{tr}}}^2+\gamma \|g_\gamma\|^2_{\mathcal H}}{\gamma} \leq \frac{  \|g-g\|^2_{\mathscr L^2_{P_{tr}}} +\gamma \|g\|^2_{\mathcal H}}{\gamma} = \|g\|^2_{\mathcal H},
				\end{align}
				or equivalently $\|g_\gamma\|_{\mathcal H}=\mathcal O(1)$ since the fixed true regression function $\|g\|_{\mathcal H}=\mathcal O(1)$. Thus, a simplified analysis shows
				\begin{align}\label{VR0}
				V_R(\rho)-\nu = &\frac{1}{\lfloor\rho n_{tr}\rfloor}\sum_{j=1}^{\lfloor\rho n_{tr}\rfloor}\hat\beta(\bm{x}_j^{tr})Y_j^{tr}-\nu \nonumber\\
				&+\frac{1}{\lfloor\rho n_{tr}\rfloor}\sum_{j=1}^{\lfloor\rho n_{tr}\rfloor}\hat\beta(\bm{x}_j^{tr})\hat g(\bm{x}_j^{tr})-\frac{1}{ n_{te}}\sum_{i=1}^{n_{te}}\hat g(\bm{x}_i^{te})
				\end{align}
				Note that the first term on the right is nothing but the $V_{KMM}$ estimator with $100\times\rho$ percent of the training data and we shall denote it as $V_{KMM}(\rho)$ without ambiguity. For the second term, assuming $\hat g=\hat g_{\gamma,\bm{X}_{NR}^{tr},\bm{Y}_{NR}^{tr}}$, is bounded by
				\begin{align}\label{VR12}
				&\frac{1}{\lfloor\rho n_{tr}\rfloor}\sum_{j=1}^{\lfloor\rho n_{tr}\rfloor}\hat\beta(\bm{x}_j^{tr})\hat g(\bm{x}_j^{tr})-\frac{1}{ n_{te}}\sum_{i=1}^{n_{te}}\hat g(\bm{x}_i^{te})\nonumber\\
				=&\frac{1}{\lfloor\rho n_{tr}\rfloor}\sum_{j=1}^{\lfloor\rho n_{tr}\rfloor}\hat\beta(\bm{x}_j^{tr})\big\langle \hat g, \Phi(\bm{x}_j^{tr})  \big\rangle_{\mathcal H}-\frac{1}{ n_{te}}\sum_{i=1}^{n_{te}}\big\langle \hat g, \Phi(\bm{x}_i^{n_{te}})\big\rangle_{\mathcal H} \nonumber\\
				=& \bigg\langle \hat g, \frac{1}{\lfloor\rho n_{tr}\rfloor} \sum_{i=1}^{\lfloor\rho n_{tr}\rfloor}\hat\beta(\bm{x}_j^{tr}) \Phi (\bm{x}_j^{tr})-\frac{1}{n_{te}}\sum_{i=1}^{n_{te}}\Phi (\bm{x}_i^{te}) \bigg\rangle_{\mathcal H} \leq \|\hat g_{\gamma,\bm{X}_{NR}^{tr},\bm{Y}_{NR}^{tr}}\|_{\mathcal H} \hat L(\hat {\bm{\beta}}),
				\end{align}
				Then, by (\ref{VR0}) and (\ref{VR12}), we have
				\begin{align}\label{fat}
				|V_R(\rho)-\nu|\leq& |V_{KMM}(\rho)-\nu|+\hat L (\hat{\bm{\beta}})(\|g_\gamma-\hat g_{\gamma,\bm{X}_{NR}^{tr},\bm{Y}_{NR}^{tr}}\|_{\mathcal H}+\|g_\gamma\|_{\mathcal H})\nonumber\\
				=& \mathcal O(n_{tr}^{-\frac{1}{2}}+n_{te}^{-\frac{1}{2}}),
				\end{align}
				following (\ref{hbound}), (\ref{learn2}) and Theorem 1 of \cite{yu2012analysis}.
			\end{proof}
			
			\begin{proof}[Proof of Proposition 2]
				If the function $g$ only satisfies the condition $\mathcal A_\infty(g, F)\triangleq \inf_{\|f\|_{\mathcal H}\leq F} \|g-f\| \leq C (\log F)^{-s}$ for some $C,s>0$, then we again follow the analysis in the proof of Proposition 1 and arrive at the decomposition in (\ref{VR0})
				\begin{align}\label{fat2}
				|V_R(\rho)-\nu|\leq& |V_{KMM}(\rho)-\nu|+\hat L (\hat{\bm{\beta}})(\|g_\gamma-\hat g_{\gamma,\bm{X}_{NR}^{tr},\bm{Y}_{NR}^{tr}}\|_{\mathcal H}+\|g_\gamma\|_{\mathcal H})\nonumber\\
				=& \mathcal O(\log \frac{n_{tr}n_{te}}{n_{tr}+n_{te}})^{-s},
				\end{align}
				which is the rate of $V_{KMM}$ by Theorem 3 of \cite{yu2012analysis}.
			\end{proof}
			
			\begin{proof}[Proof of Theorem 2]
				
				Define $\epsilon\triangleq\sup_{\theta\in \mathcal D}\bigg|V_R( \theta) - \mathbb E[l'(X^{te},Y^{te};\theta)]\bigg| $. We have
				\begin{align}
				&\mathbb E[l'(X_{te},Y_{te};\hat\theta_R)]-\epsilon\leq V_R(\hat \theta_R) \leq V_R( \theta^\star) \leq \mathbb E[l'(X_{te},Y_{te};\theta^\star)]+\epsilon.
				\end{align}
				On the other hand, we know by the triangle inequality that $\epsilon$ is bounded by
				\begin{align*}
				&\sup_{\theta\in \mathcal D}\big| \frac{1}{\lfloor\rho n_{tr}\rfloor}\sum_{j=1}^{\lfloor\rho n_{tr}\rfloor}\hat\beta(\vx_j^{tr})l'(\vx_j^{tr},y_j^{tr};\theta)-\frac{1}{n_{te}}\sum_{i=1}^{n_{te}}l(\bm{x}_i^{te}; \theta)\big|\nonumber\\
				+&\sup_{\theta\in \mathcal D}\big|\frac{1}{\lfloor\rho n_{tr}\rfloor}\sum_{j=1}^{\lfloor\rho n_{tr}\rfloor}\hat\beta(\vx_j^{tr})\hat l(\vx_j^{tr};\theta)-\frac{1}{n_{te}} \sum_{i=1}^{n_{te}} \hat l(\bm{x}_i^{te};\theta)\big|+\sup_{\theta\in \mathcal D}\big| \frac{1}{n_{te}}\sum_{i=1}^{n_{te}}l(\bm{x}_i^{te}; \theta)-\mathbb E[l(X_{te};\theta)]\big|,
				\end{align*}
				where the first term is bounded by $\mathcal O( n_{tr}^{-\frac{1}{2}}+n_{te}^{-\frac{1}{2}})$ following Corollary 8.9 in \cite{gretton2009covariate}. Moreover, the second term is also $\mathcal O( n_{tr}^{-\frac{1}{2}}+n_{te}^{-\frac{1}{2}})$ as in  (\ref{VR12}) or Lemma 8.7 in \cite{gretton2009covariate}. For the last term, due to the Lipschitz and compact assumption, it follows from Theorem 19.5 of \cite{van2000asymptotic} (see also Example 19.7 of \cite{van2000asymptotic}) that function class $\mathcal G$ is $P_{te}$-Donsker, which means that
				\begin{equation*}
				\mathbb G_{n}(\theta)\triangleq \sqrt{n_{te}}\bigg(\frac{1}{n_{te}}\sum_{i=1}^{n_{te}}l(\bm{x}_i^{te}; \theta)-\mathbb E_{\vx\sim P_{te}}[l(\vx;\theta)]\bigg)
				\end{equation*}
				converges in distribution to a Gaussian Process $\mathbb G_\infty$ with zero mean and covariance function $\text{Cov}( \mathbb G_\infty(\theta_1),\mathbb G_\infty(\theta_2))=\mathbb E_{\vx\sim P_{te}}(l(\vx;\theta_1)l(\vx;\theta_2))-\mathbb E_{\vx\sim P_{te}}l(\vx;\theta_1)\mathbb E_{\vx\sim P_{te}}l(\vx;\theta_2)$. Notice $\mathbb G_\infty$ can be viewed as random function in $C(\mathcal D)$, the space of continuous and bounded function on $\theta$. Since for any $z\in C(\mathcal D)$, the mapping $z\rightarrow \|z\|_{\infty}\triangleq \sup_{\theta\in\mathcal D}z(\theta)$ is continuous with respect to the supremum norm, it follows from the continuous-mapping theorem that $n_{te}^{\frac{1}{2}}\sup_{\theta\in \mathcal D}\big| \frac{1}{n_{te}}\sum_{i=1}^{n_{te}}l(\bm{x}_i^{te}; \theta)-\mathbb E[l(X_{te};\theta)]\big|$
				converges in distribution to $\|\mathbb G_\infty\|_\infty$ which has finite expectations based on the assumptions on $\mathcal G$ (see, e.g., Section 14, Theorem 1 of \cite{lifshits2013gaussian}). Thus, by definition of convergence in distribution, for any $\delta>0$, we can find some constant $D'$ that
				\begin{equation}\label{eprocess}
				P(\|\mathbb G_n\|_\infty>D')=P(\|\mathbb G_\infty\|_{\infty}>D')+o(1)\leq \delta+o(1),
				\end{equation}
				which means, we can find some $N$ such that when $n_{te}>N$, 
				\begin{align*}
				P_{te}\big(\sup_{\theta\in \mathcal D}\bigg| \frac{1}{n_{te}}\sum_{i=1}^{n_{te}}l(\bm{x}_i^{te}; \theta)-\mathbb E[l(X_{te};\theta)]\bigg|>n_{te}^{-\frac{1}{2}}D'\big)=P_{te}(\|\mathbb G_n\|_\infty>D')
				\leq 2\delta,
				\end{align*}
				and consequently, with probability $1-2\delta$, we have 
				\begin{equation*}
				\sup_{\theta\in \mathcal D}\big| \frac{1}{n_{te}}\sum_{i=1}^{n_{te}}l(\bm{x}_i^{te}; \theta)-\mathbb E[l(X_{te};\theta)]\big|\leq n_{te}^{-\frac{1}{2}}D'.
				\end{equation*}
				In other words, we also have
				\begin{equation*}
				\sup_{\theta\in \mathcal D}\big| \frac{1}{n_{te}}\sum_{i=1}^{n_{te}}l(\bm{x}_i^{te}; \theta)-\mathbb E[l(X_{te};\theta)]\big|=\mathcal O(n_{te}^{-\frac{1}{2}}),
				\end{equation*}
				which concludes our proof.
			\end{proof}

		\end{document}